\documentclass{article}
\usepackage{amsmath}
\usepackage{amsthm}
\usepackage{amssymb}
\newtheorem{proposition}{Proposition}
\usepackage{float}
\usepackage{algorithm}
\usepackage{setspace}

\usepackage{graphicx} 

\usepackage{multicol}
\usepackage{footmisc}
\usepackage{array, makecell}
\usepackage{caption}
\usepackage{subcaption}

\usepackage[english]{babel}
\usepackage{algpseudocode}
\usepackage{longtable}
\usepackage{booktabs}

\usepackage{float}
\usepackage{wrapfig}
\usepackage{lipsum}

\allowdisplaybreaks
\usepackage{epstopdf}

\usepackage{mathtools} 
\usepackage{graphicx}
\usepackage[colorlinks=true, allcolors=blue]{hyperref}

\usepackage{setspace}
\usepackage[export]{adjustbox}
\usepackage[
raggedleft
]{sidecap} 
\sidecaptionvpos{figure}{t}

\DeclarePairedDelimiter\floor{\lfloor}{\rfloor}

\makeindex
\usepackage{url}

\providecommand{\keywords}[1]
{
  \small	
  \textbf{\textit{Keywords---}} #1
}

\title{Collision avoidance and path finding in a robotic mobile fulfillment system using multi-objective meta-heuristics }

\begin{document}

\author{Ahmad Kokhahi$^{1}$\thanks{CONTACT Ahmad Kokhahi. Email: akokhah@clemson.edu} , Mary E. Kurz$^{1}$  \\
        \small $^{1}$Department of Industrial Engineering, Clemson University, Clemson, SC, USA}

\date{}
\maketitle
\begin{abstract}
    Multi-Agent Path Finding (MAPF) has gained significant attention, with most research focusing on minimizing collisions and travel time. This paper also considers energy consumption in the path planning of automated guided vehicles (AGVs). It addresses two main challenges: i) resolving collisions between AGVs and ii) assigning tasks to AGVs. We propose a new collision avoidance strategy that takes both energy use and travel time into account. For task assignment, we present two multi-objective algorithms: Non-Dominated Sorting Genetic Algorithm (NSGA) and Adaptive Large Neighborhood Search (ALNS). Comparative evaluations show that these proposed methods perform better than existing approaches in both collision avoidance and task assignment.         
\end{abstract}

\section{Introduction}

The rapid growth of e-commerce in recent years has significantly transformed people's shopping habits \cite{jiao2023online}. Consumers increasingly favor online shopping over in-person purchases, leading to a substantial impact on product logistics, which plays a crucial role in customer satisfaction. In addition to product quality and other factors, the timely delivery of orders has become a key determinant of customer satisfaction. Picking and replenishment tasks are responsible for 65\% of operating costs \cite{xie2021introducing}. In a conventional manual order picking system, often referred to as a picker-to-parts system, pickers dedicate 70\% of their working time to searching for items and traveling within the facility \cite{tompkins2010facilities,de2007design}. To enhance system efficiency, e-commerce companies such as Amazon have transitioned from traditional warehouses to robotic mobile fulfillment systems (RMFS). The primary distinction between these systems lies in the automation of tasks: in RMFS, robots, rather than human workers, handle the picking and replenishment of pallets. Once robots retrieve the pallets, they transport them to designated workstations, where operators are responsible for unloading. After unloading, the robots return the empty pallets to their original locations for further use. This approach significantly boosts efficiency, as it eliminates the need for pickers to move throughout the warehouse, allowing them to focus solely on picking. As a result, the overall picking productivity can be increased by up to twofold compared to traditional methods where workers must travel to gather items \cite{wurman2008coordinating}. It is possible because of path optimization. Path optimization in RMFSs enables faster and more efficient order picking by reducing travel distances and minimizing congestion. Unlike traditional warehouses with fixed routes, these systems dynamically adjust paths in real time, improving operational scalability and reducing processing times. In this paper, we propose a novel priority rule for collision avoidance focused on minimizing automated guided vehicle (AGV) energy consumption, a heuristic multi-objective task allocation algorithm to reduce collision risk and travel distance, and evaluate these methods computationally for a hypothetical robotic mobile fulfillment system. 
\section{Problem setting} \label{PS}

The proposed layout RMFS in this paper is a traditional parallel aisle warehouse. A sample of this system can be seen in figure \ref{RMFS representation}. In this figure, AGV icons show the initial locations of the AGVs, gray squares show the aisles and operator icons show the locations of workstations. 

\begin{figure}[H]
\centering
    \includegraphics[width= 3 in]{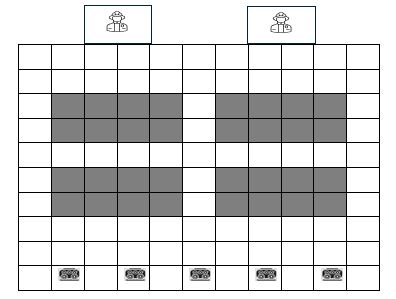}
\caption{Top view of a RMFS }
\label{RMFS representation} 
\end{figure}
The specifications of the RMFS and AGVs are as follows: 
\begin{enumerate}
  \item The RMFS spans an area of \(21\text{ } m \times 47\text{ } m\) and is represented in a square map, where each square measures \(1\text{ } m \times 1\text{ } m\).
  \item Each pod has dimensions of \(1\text{ } m \times 1\text{ } m\). A storage block is composed of \(14\) pods, arranged in a configuration of \(2 \times 7\). 
  \item The center-to-center distance between any two aisles in the warehouse is \(2\text{ } m\). Each cross-aisle has a width of \(3\text{ } m\) wide. Additionally, the vertical distance from the workstation to the nearest pod is \(2\text{ } m\).

  \item Each AGV requires \(8\) seconds for picking and \(3\) seconds for releasing the pod.

  \item In addition to accessing picking locations and workstations, the AGV can move bidirectionally at a constant speed of  \(1\) meter per second, regardless of the impact of acceleration and deceleration. Each vehicle has dimensions of \(1\text{ } m \times 1\text{ } m\).

  \item AGV charging and breakdown can not happen.

  \item eplenishment and out-of-stock situations can not happen.
  \end{enumerate}

In the RMFS, collisions can happen if there is more than one vehicle in the system. There are three types of collisions: head-on collision, cross collision and static collision. In a head-on collision, AGVs are moving directly toward each other. In cross collision, unlike head-on collision, AGVs are not moving towards each other, but their paths intersect at one point. Stay-on collision occurs when one vehicle is stationary and another moving vehicle collides with it by reaching the same position as the stationary one. These collisions can be seen in figures \ref{fig: Head_to_head}, \ref{fig:Cross} and \ref{fig:Static}, respectively.

\begin{figure}[H]
     \captionsetup{justification=raggedright,singlelinecheck=false}
     \begin{subfigure}[b]{0.3\textwidth}
         \includegraphics[width=\textwidth,left]{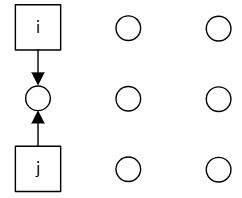}
         \caption{Head-on }
         \label{fig: Head_to_head}
     \end{subfigure}
     \hfill
     \begin{subfigure}[b]{0.3\textwidth}
         \includegraphics[width=\textwidth,left]{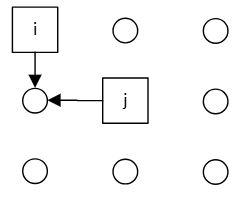}
         \caption{Cross}
         \label{fig:Cross}
     \end{subfigure}
     \hfill
     \begin{subfigure}[b]{0.3\textwidth}
         \includegraphics[width=\textwidth,left]{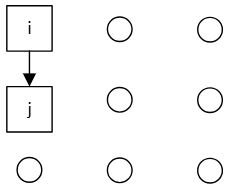}
         \caption{Stay-on}
         \label{fig:Static}
     \end{subfigure}
        \caption{Collision Types}
        \label{fig:Collision}
\end{figure}
To address these conflicts, there are some strategies that can be taken: waiting, go-away and reroute strategy. In the  waiting strategy, one vehicle waits for some time so that other vehicle can pass, then it continues on its path. This strategy is helpful to avoid cross collision. The go-away strategy involves directing AGV j to move to a neighboring node, effectively clearing the path. This maneuver ensures that AGV i can maintain its original path and continue its operations without interruption. This strategy helps optimize the flow and efficiency of both robots' movements. The reroute strategy finds a new path from source to destination that does not collide with another vehicle. The A* algorithm can find different paths from source to destination. If the shortest path finds a collision, then the next path with the lowest time can be taken to avoid collision. This strategy can be applied for all three types of collisions.

In this paper, multiple objectives are considered for optimization. The primary objective is to minimize the running time of AGVs. Additionally, we consider the minimization of energy consumption. Both objectives take into account the maximum value across all vehicles and the total summation of all values. In summary, the objectives are focused on minimizing running time and energy consumption, considering both maximum values and aggregate totals for the AGVs. The objectives can be seen below:
\begin{itemize}
    \item G1: total running time of all AGVs 
    \item G2: maximum running time of all AGVs
    \item G3: total energy consumption of all AGVs
    \item G4: maximum energy consumption of all AGVs
\end{itemize}

\section{Literature review}
The decision problems regarding RMFSs can be categorized into strategic, tactical and operational levels. The strategic level constitutes storage area dimensioning \cite{lamballais2017estimating} and workstation placement \cite{lamballais2020inventory} problems. The tactical level covers number of robots \cite{otten2019lost,yuan2017bot} and battery management problem \cite{zou2018evaluating}. In this paper, we focus on the task allocation \cite{bolu2021adaptive, dou2015genetic}, path planning \cite{kumar2018development,zhang2018collision}, and priority rules \cite{lee2019smart,turhanlar2024autonomous} which are part of the operational level.
\subsection{Task assignment}
One subject that we focus in this paper is task allocation (TA) in which tasks are assigned to robots. The majority of papers within this subject exploit heuristic algorithms to address this problem. Bolu et al.\cite{bolu2021adaptive} group orders into batches, generates order tasks based on these batches, and subsequently assigns the tasks to robots for execution.
Dou et al. \cite{dou2015genetic} propose a simulated annealing algorithm to determine the task sequence and path selection for multiple automated vehicles. Gharehgozli et al. \cite{gharehgozli2020robot} use Adaptive Large Neighborhood Search (ALNS) for task scheduling. The results show that ALNS reduced travel time by up to 30-40\% compared to standard heuristics like nearest-neighbor and farthest-neighbor. Zhang et al. \cite{zhang2019building} propose a genetic algorithm for task allocation. The proposed model is named Resource-Constrained Project Scheduling Problem with Transfer Times (RCPSPTT). This algorithm addresses task precedence, resource limitations  and variable transfer times between tasks. Zhou et al. \cite{zhou2014balanced} focuses on a new mechanism for multi-robot task allocation in intelligent warehouses. The proposed method, the Balanced Heuristic Mechanism (BHM), is designed to optimize both travel and task distribution. The BHM improves upon existing allocation methods, such as auction-based and clustering methods, by prioritizing tasks based on proximity and robot workload. Lamballais et al. \cite{lamballais2022dynamic} discuss task allocation within RMFS that continuously reallocates robots and workstations between picking and replenishment tasks based on changing demand levels. The study models the RMFS as a combination of queuing network and a Markov decision process (MDP), aiming to allocate resources efficiently during peak and non-peak demand periods.
\subsection{Path planning}
Other subject that we cover in this paper is path planning, in which efficient routes for robots to retrieve and deliver pods are determined. Sun et al. \cite{sun2021autonomous} present a path planning approach in a RMFS that focuses on interference-free scheduling of robots on bidirectional warehouses. The authors propose a modified A* algorithm for path planning. This algorithm considers three path selection rules: minimizing distance, reducing the number of turns, and selecting paths randomly to maintain variety. These paths allow robots to find alternative routes and avoid conflicts, enhancing the performance of the system. Zhang et al. \cite{zhang2018collision} introduce a collision-free route planning method for AGVs in a RMFS. This approach classifies collisions into different types and proposes specific strategies to handle each type effectively. The method uses an improved Dijkstra’s algorithm to determine the shortest routes, detect potential collisions, and choose alternative routes or modify start times to avoid conflicts.
\subsection{Priority rules}
In cases where multiple vehicles are involved in a collision, establishing the priority order of the vehicles is essential. Many papers only use predetermined rules as priority rules. For example, one rule is that the AGV having the earliest arrival time has the highest priority \cite{lu2023automated,li2024reinforcement}. One other rule is to consider the status of the vehicle, meaning that the vehicle which is loaded has the higher priority. One type of rule tries to prioritize the AGVs based on their state. For example, the system assigns priority to AGVs based on task type, rack frequency, and the relevance of goods, creating an 11-level hierarchy to guide conflict resolution \cite{cai2021collaborative}. There are papers in which priority rules are extensively studied. Lee et al. \cite{lee2019smart} establish a priority rule for mobile robots in a warehouse system to manage conflicts and avoid collisions. Each robot's priority is calculated using a weighted formula based on its task assignment, rack-lifting status, estimated path time, and time already spent on the current task. Robots with higher priority levels are allowed to proceed on their planned paths, while those with lower priority may need to adjust, wait, or detour to resolve potential conflicts. This rule helps streamline task fulfillment and enhance operational efficiency by minimizing disruptions and optimizing robot paths in real-time. Demesure et al. \cite{demesure2017decentralized} compute the priority using an index called Individual Performance (IP). This index determines the priority of each robot based on its current state, such as its estimated travel time to a resource and how much it aligns with the system’s specifications. If a robot cannot satisfy the system’s specifications, the robot is given the highest priority (IP = 1), which allows it to schedule its product as soon as possible.

\subsection{Battery management problem}
The robots in RMFS are battery powered. Zou et al. \cite{zou2018evaluating} investigate battery management in RMFS, emphasizing its importance due to the impact of battery depletion on system performance and costs. Battery depletion reduces the availability of robots, hence decreasing the system's throughput capacity. The study investigates three primary battery recovery strategies: plug-in charging, inductive charging, and battery swapping. Each strategy has its own advantages and disadvantages. Inductive charging provides the best throughput time, as robots can charge while waiting at workstations, eliminating downtime for dedicated charging. Battery swapping allows robots to quickly replace depleted batteries, offering higher throughput than plug-in charging, though it increases costs. Plug-in charging, while slower in terms of throughput, often results in lower annual costs compared to battery swapping. In other research, Chen et al. \cite{chen2025role} evaluate three dynamic priority policies to optimize order handling and minimize energy use. They propose strategies like real-time data integration and optimal robot deployment. Additionally, it suggests an ideal layout for reduced travel distances, balancing operational demands with energy efficiency.

\subsection{Contributions and paper structure}
Contributions of this paper are listed as follows:
\begin{itemize}
    \item In this paper, a mixed integer formulation of the problem is provided. Then, it  solved for a small sized problem  and we prove that this is an NP-hard problem.
    
    \item This paper introduces a novel priority rule addressing collision avoidance, with a focus on minimizing the energy consumption of AGVs. 
    \item This paper proposes a heuristic multi-objective task allocation algorithm aiming to minimize possibility of collision and total traveled distance by AGVs.

    \item This paper evaluates the proposed algorithms and priority rule and compare them with existing methods via computational studies.
\end{itemize}
The rest of this paper is organized as follows: \ref{PF} provides the mathematical formulation
for the problem and an efficient approach to solve the small size problems. \ref{Md}
outlines the methodology used in this study. \ref{resul} begins by presenting the metrics
used to evaluate the algorithms and concludes by discussing the results. \label{MI} provides recommendations for managers based on the study’s findings. Finally, \ref{MI} summarizes the key results and conclusions of the paper.

\section{Problem formulation and mathematical model}\label{PF}
\subsection{Variable definition}
In this section, the mathematical formulation of the problem is presented. We provide a formulation for the problem of multiple AGVs considering time and energy consumption(MATEC). The notations in this formulation are defined as follows.   \begin{longtable}{>{\bfseries}p{3cm}p{12cm}}\caption{Sets and parameters}
\label{tab:S_P}\\
\multicolumn{2}{c}{\large\textbf{Sets and Parameters}} \\
\toprule
$A$         & Set of AGVs, $a \in A$ \\
$L$         & Set of locations $\lambda, \lambda \in L $ \\
$W$         & Set of workstations, $w \in W \subset L$ \\

$K$       & Set of task locations, $k \in K \subset L$ \\
$S$       & Set of loading status values where 0 means unloaded and 1 loaded, $S \in \{0,1\}$ \\
$T$         & Time horizon, $t \in \{1,2,...,T\}$ \\
$P$       & Set of arcs in space-time network in which AGV moves from location $\lambda$ to location $\lambda'$ with loading status $s$ at time $t$ , $(\lambda,\lambda',s,t) \in P$ \\
$P_{M}^{W}$       & Set of arcs leading to a workstations (the destination($\lambda'$) is a workstation while the source($\lambda$) is not. ), $P_{M}^{W} \subset P$ \\
$P_{M}^{O}$       & Set of arcs leading to storage locations(the destination($\lambda'$) is a storage location while the source($\lambda$) is not. ), $P_{M}^{O} \subset P$ \\
$P_{W}^{W}$       & Set of waiting arcs in workstations(the destination($\lambda'$) and the source($\lambda$) are same and it is a workstation. ), $P_{W}^{W} \subset P$ \\
$P_{W}^{O}$       & Set of waiting arcs in storage locations(the destination($\lambda'$) and the source($\lambda$) are same and it is a storage location. ), $P_{W}^{O} \subset P$ \\

$\tau_w$         & Time for unloading at workstations \\
$\tau_p$         & Time for loading at storage locations \\

$I$       & Set of initial locations of AGVs\\
$I^{a}$       & Initial location of AGV $a$, $I^{a} \in I$\\

\midrule
\multicolumn{2}{c}{\large\textbf{Decision Variables}} \\
\midrule
$y_{a,\lambda,\lambda',s,t}$        & 1 if AGV $a$ moves from location $\lambda$ to location $\lambda'$ with loading status $s$ at time $t$ and $0$ otherwise \\
$r_k^{a}$ & $1$ if AGV $a$ is assigned to task $k$ and $0$ otherwise\\
$b_{(a,\lambda,\lambda',s,t), k}$     & $1$ if arc $(a,\lambda,\lambda',s,t)$ is associated with task $k$ and $0$ otherwise\\
$S_{k_1,k_2}^a$  & $1$ if AGV $a$ processes task $k_1$ immediately before $k_2$ $0$ otherwise \\
$x^{s}_{k}$       & Time that an AGV enters task location $k$ to lift the pod  \\
$x^{b}_{k}$       & Time that an AGV enters task location $k$ to return the unloaded pod\\
\bottomrule
\end{longtable}
\subsection{Model establishment}
There are four different objective functions that are considered in this paper, as mentioned in section \ref{PS}. Their formulations are as follows:

\begin{itemize}
    \item Minimizing total running time of all AGVs 
    \begin{equation}
\min \left( \sum_{(\lambda, \lambda', s, t) \in \mathcal{P}}  y_{a, \lambda, \lambda', s, t} \cdot (t+1) \right)
\label{eq:objective1}
\end{equation}
    
    \item  Minimizing maximum running time of all AGVs 
    \begin{equation}
\min \left( \max_{a \in \mathcal{A}} \sum_{(\lambda,\lambda',s,t) \in \mathcal{P}} y_{a,\lambda,\lambda',s,t} \cdot (t+1) \right)
\label{eq:objective2}
\end{equation}
    \item Minimizing total energy consumption of all AGVs
    \begin{equation}
\min \left( \sum_{(\lambda, \lambda', s, t) \in \mathcal{P}} y_{a, \lambda, \lambda', s, t} \cdot \mathbf{1}_E(a) \right)
\label{eq:objective3}
\end{equation}
    \item Minimizing maximum energy consumption of all AGVs
    \begin{equation}
\min \left( \max_{a \in \mathcal{A}} \sum_{(\lambda, \lambda', s, t) \in \mathcal{P}} y_{a, \lambda, \lambda', s, t} \cdot \mathbf{1}_E(a) \right)
\label{eq:objective4}
\end{equation}
\end{itemize}

The  function $\mathbf{1}_E(a)$ in equations \ref{eq:objective3} and \ref{eq:objective4} determines the energy consumption of AGVs based on the their loading status. The function can be seen below.  
\[
\mathbf{1}_E(a) =
\begin{cases}
E_L, & \text{if } a \text{ is loaded} \\
E_{UL}, & \text{if } a \text{ is not loaded}
\end{cases}
\]

The constraints of the problem can be seen below. 
\begin{align}
&\displaystyle \sum_{a \in\mathcal{A} } r_{k}^{a} = 1 , \forall\, k \in\mathcal{K} \label{eq:25} \\
&\displaystyle \sum_{a \in\mathcal{A} } \sum_{k_1 \in\mathcal{K} \cup \mathcal{I}^{a}, k_1 \neq k_2}  S_{k_1,k_2}^{a} = 1   , \quad \forall\, k_2 \in\mathcal{K} \label{eq:28} \\
&\displaystyle \sum_{a \in\mathcal{A} } \sum_{k_2 \in\mathcal{K} \cup \mathcal{I}^{a}, k_2 \neq k_1}  S_{k_1,k_2}^{a} = 1   , \quad \forall\, k_1 \in\mathcal{K} \label{eq:29} \\
&\displaystyle \sum_{a \in\mathcal{A} } S_{k_1,k_2}^{a} \leq r_{k_1}^{a}  , \forall\, k_1,k_2 \in\mathcal{K} \label{eq:26} \\
&\displaystyle \sum_{a \in\mathcal{A} } S_{k_1,k_2}^{a} \leq r_{k_2}^{a}  , \forall\, k_1,k_2 \in\mathcal{K} \label{eq:27} \\
&\displaystyle \sum_{ ({I}^a,\lambda',0,0) \in \mathcal{P} } y_{a,{I}^a,\lambda',0,0} = 1 , \forall\, a\in\mathcal{A} \label{eq:2} \\
& \displaystyle \sum_{ (\lambda,k,0,t) \in\mathcal{P_{M}^{O}}} y_{a,\lambda,k,0,t} = 1 ,\quad \forall\, k \in \mathcal{K}, \forall\, a \in\mathcal{A} \label{eq:3}  \\
& \displaystyle \sum_{ (\lambda,k,1,t) \in\mathcal{P_{M}^{O}}} y_{a,\lambda,k,1,t} = 1 ,\quad \forall\, k \in \mathcal{K}, \forall\, a \in\mathcal{A} \label{eq:4}\\
& \displaystyle \sum_{ \lambda ,\lambda',s,t \in\mathcal{P}} y_{a,\lambda,\lambda',s,t} = 1 , \forall\, a\in\mathcal{A} , \forall\, t \in\mathcal{T} \label{eq:6}\\
& \displaystyle  y_{a,\lambda,\lambda',s,t} \leq \sum_{ \lambda'',\lambda,|s-1|,t-1 \in\mathcal{P}} y_{a,\lambda'',\lambda,|s-1|,t-1} ,\quad \forall\, a\in\mathcal{A} , \forall\ s \in \mathcal{S}, \forall\ (\lambda,\lambda',s,t) \in \mathcal{P_M^{W}}, \forall\ t,t-1 \in \mathcal{T}   \label{eq:7}\\ 
& \displaystyle  y_{a,\lambda,\lambda',s,t} \leq \sum_{ \lambda'',\lambda,0,t-1 \in\mathcal{P}} y_{a,\lambda'',\lambda,0,t-1} ,\quad \forall\, a\in\mathcal{A} , \forall\ \lambda,\lambda',s,t \in \mathcal{P_M^{O}}, \forall\ t,t-1 \in \mathcal{T} \label{eq:8}\\
& \displaystyle  y_{\lambda,\lambda' ,l,s,t} \leq \sum_{ \lambda'',\lambda,s,t-1 \in\mathcal{P}} y_{a,\lambda'',\lambda,s,t-1} ,\quad \forall\, a\in\mathcal{A} , \forall\ \lambda,\lambda',s,t \in \mathcal{P}, \forall\ t,t-1 \in \mathcal{T}  \label{eq:9}\\
& \displaystyle \sum_{ a\in\mathcal{A} }\sum_{t= t_1}^{t_1 + \tau_w} \sum_{ (\lambda,w,1,t)\in\mathcal{P_M^{W}} } y_{a,\lambda,w,1,t} \leq 1 ,  \forall\, t_1 \in\mathcal{T} , \quad \forall\, w \in \mathcal{W} \label{eq:10}\\
& \sum_{t= t_1}^{t_1 + \tau_w} \sum_{ (w,w,s,t)\in\mathcal{P_W^{W}} } y_{a,w,w,s,t} \geq \tau_{w} \cdot y_{a,\lambda,w,s',t-1},  \quad \forall\, t_1 \in\mathcal{T} , \forall\,a\in\mathcal{A},\forall\ w \in W, \forall\ (w,w,s',t-1) \in \mathcal{P_{M}^{W}} \label{eq:11}\\
& \sum_{t= t_1}^{t_1 + \tau_p} \sum_{ (k,k,s,t)\in\mathcal{P_W^{P}} } y_{a,k,k,s,t} \geq \tau_p \cdot y_{a,\lambda,k,s',t-1},  \quad \forall\, t_1 \in\mathcal{T} , \forall\,a\in\mathcal{A}, \forall\ (\lambda,k,s',t-1) \in \mathcal{P_{M}^{P}} \label{eq:12}\\
& x^{s}_{k} = \sum_{ t \in \mathcal{T}} \sum_{(\lambda,k,1,t) \in \mathcal{P_{M}^{O}}} y_{a,\lambda,k,1,t} \cdot (t+1) ,\quad \forall k \in \mathcal{K} \label{eq:14}\\
& x^{b}_{k} = \sum_{ t \in \mathcal{T}}  \sum_{(\lambda,k,0,t) \in \mathcal{P_{M}^{O}}} y_{a,\lambda,k,0,t} \cdot (t+1),\quad \forall k \in \mathcal{K} \label{eq:15}\\
& x^{b}_{k} \geq x^{s}_{k}, \quad \forall k \in \mathcal{K} \label{eq:16}\\
& x^{b}_{k_1} \leq x^{s}_{k_2} + M\cdot (1 - S_{k_1,k_2}^{a}), \quad \forall k_1,k_2 \in \mathcal{K}, \forall a \in \mathcal{A}\label{eq:36}\\
 &\sum_{ (\lambda,w,s,t) \in \mathcal{P_M^{W}}} b_{(a,\lambda,w,s,t), k} = r_{k}^{a}, \quad \forall k \in \mathcal{K},\forall\, a \in \mathcal{A} \label{eq:13}\\
 & y_{a,\lambda,w,s,t}\cdot (t+1) \geq (x^{s}_{k} - M (1 - b_{(a,\lambda,w,s,t) , k})), \quad \forall(\lambda,w,s,t) \in \mathcal{P_{M}^{W}}, \forall k \in \mathcal{K}, \forall a \in \mathcal{A} \label{eq:19}\\
 & y_{a,\lambda,w,s,t}\cdot (t+1) \leq (x^{b}_{k} + M (1 - b_{(a,\lambda,w,s,t)  , k})), \quad \forall (\lambda,w,s,t) \in \mathcal{P_{M}^{W}}, \forall k \in \mathcal{K}, \forall a \in \mathcal{A}\label{eq:20}\\
 &\sum_{(\lambda,\lambda',s,t) \in \mathcal{P}} y_{a,\lambda,\lambda',s,t}  + \sum_{(\lambda,\lambda',s',t') \in \mathcal{P}} y_{a',\lambda,\lambda',s',t'} \leq 1,\quad  \forall a,a' \in \mathcal{A}, a \neq a' \label{eq:17}\\
 &\sum_{(\lambda,\lambda',s,t) \in \mathcal{P}} y_{a,\lambda,\lambda',s,t} + \sum_{(\lambda,\lambda',s',t') \in \mathcal{P}} y_{a',\lambda,\lambda',s',t'} \leq 1,\quad  \forall a,a' \in \mathcal{A}, a \neq a' \label{eq:18}\\
 & y_{a,\lambda,\lambda',s,t} \in \{0,1\} \\
 & b_{(a,\lambda,\lambda',s,t), k} \in \{0,1\} \\
 & x^{s}_{k} \geq 0,\quad \forall k \in \mathcal{K} \label{eq:162}\\
 & x^{b}_{k} \geq 0,\quad \forall k \in \mathcal{K} \label{eq:163}\\
 & r_{k}^{a} \geq 0,\quad  \forall k \in \mathcal{K}, \forall a \in \mathcal{A} \label{eq:164}\\
 & S_{k_{1},k_{2}}^{a} \geq 0, \quad \forall k_{1},k_{2} \in \mathcal{K}, \forall a \in \mathcal{A} \label{eq:165}
\end{align}
In this formulation, constraints \ref{eq:25} to \ref{eq:27} determine the assignment of tasks to AGVs and also determine in which order these tasks are assigned to AGVs. Constraint \ref{eq:25}  ensures that each task must be allocated to exactly one AGV. Constraints \ref{eq:28} and \ref{eq:29} ensure that each task must come exactly after one task and before other task. These constraints determine the order of tasks for each AGV. Constraints \ref{eq:26} and \ref{eq:27} ensure that constraints \ref{eq:28} and \ref{eq:29} are focused on one AGV.

Constraint \ref{eq:2} ensures that each vehicle starts from the designated initial location. Constraints \ref{eq:3} and \ref{eq:4} are about stating that each task must be visited twice, once for loading the pod and once for returning unloaded pod. Constraint \ref{eq:6} ensures that each AGV at any time must be at one location.

Constraints \ref{eq:7}, \ref{eq:8}, and \ref{eq:9} ensure that an AGV can occupy a given arc at time \textit{t} only if it was located on a neighboring arc leading into it at time \textit{t}-1, thereby preserving the temporal flow continuity across the space-time network. Constraint \ref{eq:7} focuses on the nodes leading to workstations. When an AGV enters a workstation, the loading status of that AGV will change. Constraint \ref{eq:8} focuses on the nodes leading to task pods. AGVs are always empty before entering task locations. 
Constraint \ref{eq:9} focuses on the task pods. For these arcs, the status of AGVs does not change.

Constraint \ref{eq:10} ensures that when an AGV receives service in workstation, no other AGV can not enter the workstation. Constraint \ref{eq:11} ensures that when an AGV enters a workstation for service, it has to stay in that workstation until it receives the service it needs. 

Constraint \ref{eq:12} ensures that when an AGV enters a task pod for service, it has to stay in that location until it receives the service. Constraint \ref{eq:14} finds the time that AGVs enter task pods for lifting them. Constraint \ref{eq:15} finds the time that AGVs enter task pods for returning unloaded pods. Constraint \ref{eq:16} ensures that the time of returning the unloaded pod is after the time that the pod is picked up. 

Constraint \ref{eq:36} ensures that any AGV, after returning the unloaded pod to its location, will begin traveling to the next task in its list. 

Constraints \ref{eq:13},\ref{eq:19} and \ref{eq:20} ensure that after an AGV load a task pod, it must visit a workstation and receive service and then it can return the unloaded pod to its first place.

 Constraints \ref{eq:17} and \ref{eq:18} ensure that there will be no collisions between AGVs in the system. Constraint \ref{eq:17} guarantees that no two AGVs occupy the same node at time any time $t$. Constraint \ref{eq:18} guarantees that no two vehicles have head-to-head collision in the system.

Constraints \ref{eq:162} to \ref{eq:165} are sign restrictions for the decision variables.

\begin{proposition}
\label{prop:nphard}
Problem MATEC is NP-hard.
\end{proposition}

\begin{proof}
By relaxing two sets of constraints—(i) each workstation can serve only one pod within a limited time window (constraints \ref{eq:10} and \ref{eq:11}), and (ii) automated guided vehicles (AGVs) must avoid collisions (constraints \ref{eq:17} and \ref{eq:18})—the problem reduces to a simpler form. In this simplified version, each AGV starts from its initial location, visits a subset of tasks assigned to it, transports the associated pod to a nearby workstation for service, returns the pod to its original location, and finally returns to its depot. The travel costs between AGVs and tasks, as well as between tasks, are known. Moreover, each task incurs a fixed cost corresponding to the round-trip movement between the pod’s location and the closest workstation.

This formulation corresponds to a Multi-Depot Traveling Salesman Problem (MDTSP), where AGVs serve as salesmen, each originating from a unique depot, and tasks act as customer nodes. The MDTSP is a well-known NP-hard problem. Hence, the original problem MATEC, being a generalization of this simplified version, is also NP-hard.
\end{proof}

Although solving the full problem optimally is computationally intractable for large instances, small-scale problems can be addressed by leveraging the MDTSP structure. Specifically, for a warehouse of size $10 \times 10$ meters with 5 tasks and 3 AGVs (as shown in Figure~\ref{MDTSP_fig}), an exhaustive enumeration strategy can be used.

\begin{figure}[H]
\centering
\includegraphics[width=5in]{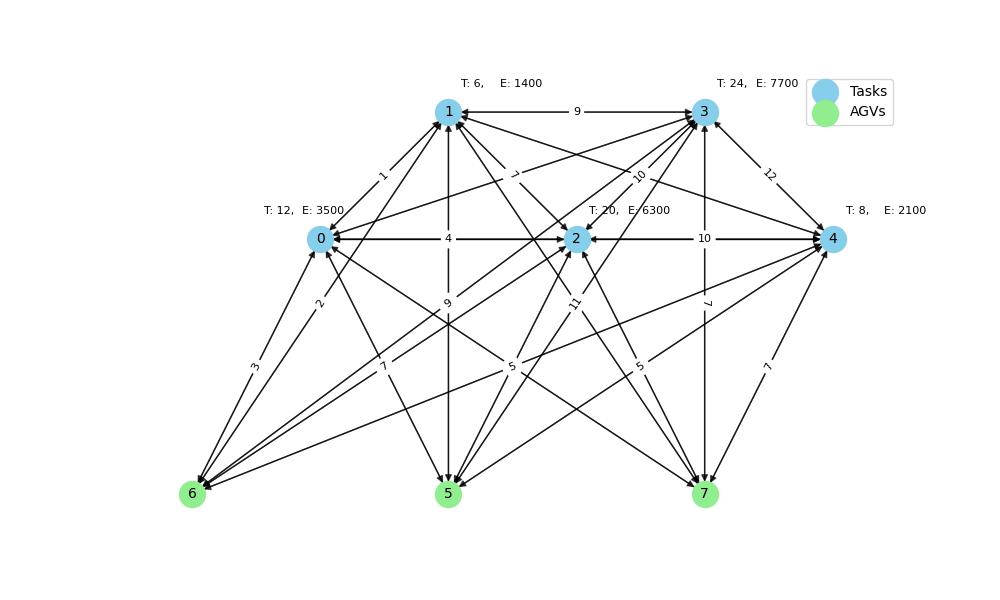}
\caption{MDTSP representation: Nodes 0–4 represent tasks. 'E' and 'T' denote the fixed energy and time costs of moving each pod to and from the workstation. Nodes 6–8 correspond to AGV depots. Arc weights represent travel time only, for clarity.}
\label{MDTSP_fig}
\end{figure}

For this instance, there are 720 feasible task-to-AGV assignment permutations, assuming that each AGV is assigned at least one task. Among these, 22 solutions are non-dominated with respect to the multiple objective criteria considered, thus forming the initial Pareto front.

Each scenario in the Pareto front is then evaluated using the full version of the Problem, with task assignments and their service sequences fixed. This allows assessment of the true (non-relaxed) performance. The results from these 22 instances serve as a reference set.

Next, the full set of 720 MDTSP solutions is screened to identify additional scenarios that are not dominated by the 22 reference solutions. A total of 42 such scenarios are identified. These are then evaluated by solving the full version of the Problem, with their task sequences fixed as in the MDTSP solution. In this case, the $r_{k}^{a}$ variables are set according to the scenario.

Combining these 42 solutions with the initial 22 yields a set of 64 candidate solutions. Upon Pareto filtering this final set, 16 scenarios remain that are non-dominated under the original problem's objectives. These represent the final Pareto front and provide a comprehensive set of trade-off solutions for the small-scale problem instance.    

\begin{algorithm}[H]
\caption{Pareto Front Identification via Enumerative MDTSP-based Refinement}
\label{alg:pareto_identification}

\textbf{Input:} Set of tasks $\mathcal{K}$, set of AGVs $\mathcal{A}$, warehouse layout \\
\textbf{Output:} Final Pareto front of non-dominated solutions $\mathcal{P}_{\text{final}}$

\begin{algorithmic}[1]
\State Compute pairwise travel times between all nodes (AGVs, tasks, workstations)
\State Compute fixed cost for each task (e.g., pod movement to and from workstation)
\State Generate all valid task-to-AGV assignment permutations where each AGV has at least one task
\State Let $\mathcal{S}$ be the set of all such feasible assignment scenarios

\State Initialize empty list of costs: $C = \{\}$

\For{each scenario $s \in \mathcal{S}$}
    \State Assign tasks to AGVs according to scenario $s$
    \State Determine sequence of tasks per AGV
    \State Compute total cost of scenario $s$ using MDTSP formulation (travel + fixed costs)
    \State Store cost and solution for $s$ in $C$
\EndFor

\State Identify non-dominated scenarios from $C$ and store them as $\mathcal{P}_{\text{init}}$

\State Initialize empty set $\mathcal{E}_{\text{init}}$
\For{each scenario $s \in \mathcal{P}_{\text{init}}$}
    \State Fix task assignment and order
    \State Solve full version of Problem MATEC for scenario $s$
    \State Store result in $\mathcal{E}_{\text{init}}$
\EndFor

\State Compare all scenarios in $\mathcal{S}$ to $\mathcal{P}_{\text{init}}$
\State Identify additional scenarios that are not dominated $\rightarrow \mathcal{C}_{\text{add}}$

\State Initialize empty set $\mathcal{E}_{\text{add}}$
\For{each scenario $c \in \mathcal{C}_{\text{add}}$}
    \State Fix task assignment and order
    \State Solve full version of Problem MATEC for scenario $c$
    \State Store result in $\mathcal{E}_{\text{add}}$
\EndFor

\State Combine $\mathcal{E}_{\text{init}}$ and $\mathcal{E}_{\text{add}}$
\State Identify non-dominated solutions in combined set $\rightarrow \mathcal{P}_{\text{final}}$
\State \Return Final Pareto front $\mathcal{P}_{\text{final}}$
\end{algorithmic}
\end{algorithm}
\section{Methodology}\label{Md}
\subsection{Path routing}
One of the main issue in any RMFS system is path routing. It is crucial to determine the best path that each robot can take to reach its destination. The popular algorithms for regarding this matter are Dijkstra's and A* algorithms. Dijkstra's algorithm uses a breadth-first approach to find the shortest path from a given starting node to all other nodes in a weighted, directed graph. The A* algorithm is a heuristic algorithm based on Dijkstra's algorithm but computes the shortest tree between two given nodes (as opposed to all potential destination nodes). Hence, A* has better performance than Dijkstra's algorithm. In this paper we propose a modified A* algorithm for the path routing.

\begin{figure}[H]
\centering
    \includegraphics[width= 5 in]{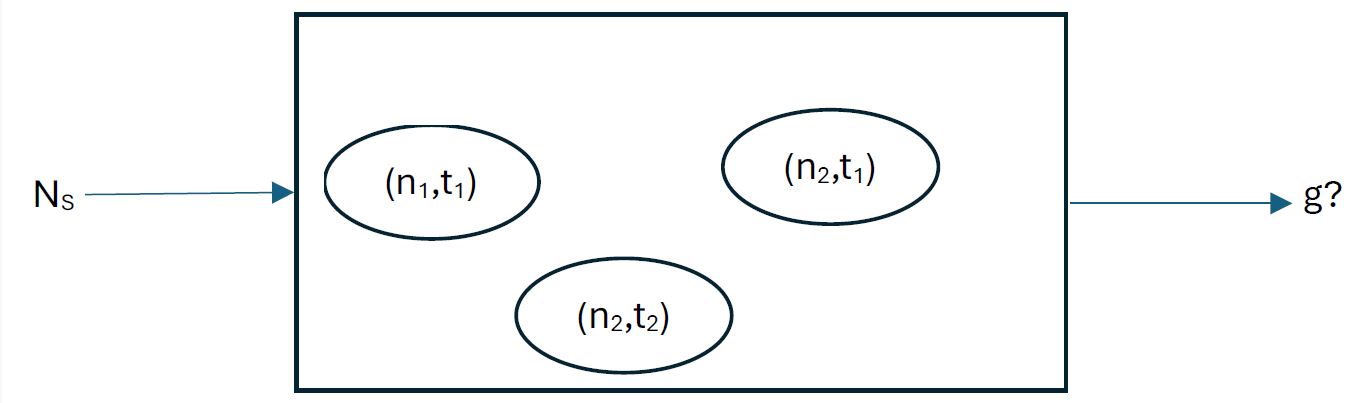}
\caption{A* algorithm }
\label{A* representation} 
\end{figure}

\subsubsection{Modified A* algorithm}
A* algorithm evaluates each candidate node n and time t combination (\(n,t\)) using a cost function, \(f(n,t)\), defined as:
\begin{equation}\label{equal:1}
f((n,t)) = g((n,t)) + h((n,t))    
\end{equation}
where \(g(n,t)\) is the cost of the path from starting point to node \((n,t)\), and \(h(n,t)\) is a heuristic estimate of the cost from node \((n,t)\) to the goal location. In node \((n,t\)), \(n\) is the location of the node and \(t\) is time of the node. In this paper, the purpose is minimizing the running time of vehicles. Hence, \(g(n,t)\) and \(h(n,t)\) are defined as:
\begin{equation}\label{equal:2}
g((n,t)) = t - ST 
\end{equation}
\begin{equation}\label{equal:3}
h((n,t)) = Manhattan\_distance(n,N_D.location)  
\end{equation}
In equation \ref{equal:2}, \(ST\) is the time of the starting node at the beginning of the decision epoch. In equation \ref{equal:3}, \(n\) is the location of current node and \(N_D.location\) is the location of destination node.

The traditional A* algorithm can not consider a collision with other vehicles. Hence, it is not guaranteed to the find the best path in a RMFS while there is more than one vehicle. The proposed algorithm can be seen in algorithm \ref{algo:a*}. In this algorithm, \(r_{s,t}\) is the space-time network in which all the paths of all vehicles are kept. In the proposed modified A* algorithm, the nodes have location and time. The implementation is similar to traditional A* algorithm, except the collision with other vehicles.  If a collision occurs outside an aisle, then that node will not be expanded. The other type of the collision occurs insides aisles. In these collisions, the penalty time is variable. In this paper, we calculate the exact time lost by the collision.

Considering collision inside aisles in algorithm \ref{algo:CRA} is discussed with an example. In figure \ref{fig:Cia}, \(V(A)\) and \(V(B)\) represent AGV \(A\) and \(B\), respectively. Also \(DES(B)\) represents the final destination of AGV \(B\). The arrows show the direction of the movement of the AGVs. The algorithm wants to find the best path for AGV \(A\). If a collision occurs inside the aisle, the vehicle has to take a step back. In figure \ref{col: T}, it can be seen that AGV \(A\) collides with AGV \(B\) at time \(T\), hence it must take a step back (figure \ref{col: T+1}). It is important to know that AGV \(A\) can collide with vehicle AGV \(B\) at time \(T+1\) since they are inside an aisle and the collision between them is an opposite conflict. Hence, AGV \(A\) must retreat further and further until either AGV \(B\) reaches its destination or leaves the aisle. In this example, the collision is resolved at time \(T+2\), since AGV \(B\) reaches its destination. Now, the lost time due to collision must be updated. In algorithm \ref{algo:CRA}, \(added\_time\) represents the lost time. In figure \ref{fig:Cia}, the \(added\_time\) is \(2\) seconds, since the difference time between the time in which two vehicles first experience collision and the time that collision is resolved is \(2\) seconds (\((T+2) - T = 2\)). This calculated lost time must be multiplied by two, because AGV \(A\) moves two seconds to reach AGV \(B\) and the same time to get back to that node. Finally, the node with the updated time is added to the algorithm, as seen in line 6-9 of algorithm \ref{algo:CRA}. If a collision occurs with a new AGV inside an aisle, like AGV \(C\) in figure \ref{fig:Cia}, then a deadlock has occurred. To address this collision, a new path is computed after waiting for \(added\_time\times2\). In this case, the total running of the new path is compared with the first path and the path with lower time is chosen.   

\begin{figure}[H]
     \captionsetup{justification=raggedright,singlelinecheck=false}
     \begin{subfigure}[b]{0.5\textwidth}
         \includegraphics[width=\textwidth,left]{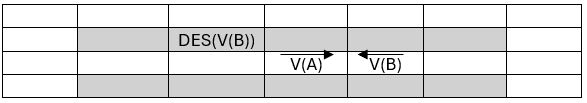}
         \caption{time T }
         \label{col: T}
     \end{subfigure}
     \hfill
     \begin{subfigure}[b]{0.5\textwidth}
         \includegraphics[width=\textwidth,left]{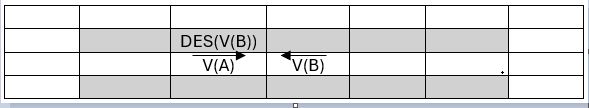}
         \caption{time T+1}
         \label{col: T+1}
     \end{subfigure}
     \hfill
     \begin{subfigure}[b]{0.5\textwidth}
         \includegraphics[width=\textwidth,left]{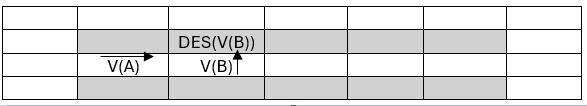}
         \caption{time T+2}
         \label{col: T+2}
     \end{subfigure}
        \caption{Collision inside aisle}
        \label{fig:Cia}
\end{figure}

\begin{algorithm}[H]
\caption{modified A* algorithm}\label{algo:a*}
\textbf{Input} \text{start node\( N_{S}\) , goal node location\( N_{D}.location\) , start time$ST$ },\\ \text{\hspace{32pt}space\_time network routes (\(r_{s,t}\))} \\
\textbf{Output} \text{updated space\_time network routes (\(r_{s,t}\))}

  \begin{algorithmic}[1]
    \State generate an empty set of expanded nodes; \(E = \{\}\)
    \State generate an empty set of closed nodes; \(CL = \{\}\)

    \State \text{current node }$= N_{S}$
    \State $camefrom = \{\}$
    
    \State \text{add current node to \(E\) }
    
    \State compute evaluation function of current node

    \While{current node.location $\neq N_{D}.location$ }
    \State $    Nei(cu) = $ \text{all the neighbors of the current node}
    \For{$j\in Nei(cu)$}
    \State \text{add the node $j$ to list $E$}
    \State $camefrom(j) =$ current node
    \State compute evaluation function of node $j$
    \EndFor

    \State \text{remove the current node from $E$}
    \State \text{add the current node to $CL$}
    \If { current node collide with any vehicle $V_j$ \textbf{and} collision is inside aisle }
    
    \State \text{$E= $   Implement algorithm \ref{algo:CRA}}

    \EndIf
    
    \State current node = \text{find the node in $E$ with minimum evaluation function}

    \EndWhile

  \end{algorithmic}
\end{algorithm}

\begin{algorithm}[H]
\caption{Collision Resolve algorithm}\label{algo:CRA}
\textbf{Input} \text{current node \((c\_n , c\_t)\), \(camefrom\) , evaluation function,} \\ \text{\hspace{32pt}expanded list\(E\)}  \\
\textbf{Output} \text{updated evaluation function, updated expanded list\(E\)} 

  \begin{algorithmic}[1]
    \State $added\_time = 0$
    \While{$True$}
    
    \State $added\_time = added\_time + 1$
    \State $(c\_n,c\_t) = camefrom((c\_n,c\_t))$
    
    \If { $c\_n$ not inside aisle  \textbf{or} no collision }
    \State $((c\_n,t^{'}_n)) = 
    (c\_n,c\_t + added\_time\times2)$
    
    \State \text{add $((c\_n,t^{'}_n))$ to set \(E\)}
    \State $f((c\_n,t^{'}_n)) = g((c\_n,t^{'}_n)) + h((c\_n,t^{'}_n))$

    \State break
    \EndIf
    \If{$c\_n$ collide with vehicle $V_k \neq V_j$}
    \State \text{Run a new algorithm after waiting for \(added\_time \times2\).}
    \EndIf

    \EndWhile

  \end{algorithmic}
\end{algorithm}

\subsection{Priority rule}
A priority rule is a guideline used to determine which object or vehicle has the right-of-way in situations where paths might intersect. To the best knowledge of authors, no research has been conducted yet to examine the impact of priority rules on running time of the AGVs and energy consumption. In this paper, we aim to determine priority for AGVs based on the minimizing energy consumption.

As discussed in the literature review, the majority of priority rules for AGVs are influenced by their arrival times and operational statuses. Collision avoidance strategies can be considered a means to minimize energy consumption and reduce the running time of AGVs. This paper addresses the collision avoidance problem with a specific focus on optimizing energy efficiency for AGVs.

The amount of energy that an AGV consumes is a factor of its load status. A loaded AGV uses 500 joules to travel one meter, while an unloaded AGV requires 200 joules for the same distance \cite{zou2018evaluating}.
The proposed priority rule is explained with an example. Assume two AGVs \(i\) and \(j\) collide with each other. There are two possible scenarios for these vehicles, either AGV \(i\) goes before \(j\) or AGV \(j\) goes before \(i\). In the first scenario, we implement the modified A* algorithm for AGV \(j\) with and without considering AGV \(i\). The difference between these paths are decisive factors. Hence, the algorithm calculates the difference  between these two paths first. Then, it calculates the difference between energy consumption. This is the total wasted energy consumption due to the collision. The same process is applied for the second scenario. Finally, we compare the total wasted energy of the two scenarios and choose the one with lowest value. In case that there is more than one scenario with lowest energy consumption, the scenario with lowest wasted time is chosen.

This priority rule algorithm is explained in algorithm \ref{algo:PRA}. In this algorithm, in the first line, all the permutation of AGVs are considered. For example, if three AGVs \(V_1,V_2,V_3\) are involved in the collision, then the set of all the possible priorities are \( \{(V_1,V_2,V_3) , (V_1,V_3,V_2) , (V_2,V_1,V_3) , (V_2,V_3,V_1) , (V_3,V_1,V_2) , (V_3,V_2,V_1)\}\), where in \((V_1,V_2,V_3)\), the AGV \(V_1\) has the highest priority, \(V_2\) has the second highest priority and \(V_3\) has the lowest priority. For each priority scenario, we apply the A* algorithm and compute the wasted energy consumption for each AGV in lines 6-10 of the algorithm \ref{algo:PRA}. Finally, we find the AGV with lowest energy consumption (lines 12-14 of algorithm \ref{algo:PRA}). It is important to mention that companies usually plan the AGVs so that not many of them have collisions inside aisles. In our implementations, we only observe a maximum of \(4\) vehicles experiencing collisions.    

\begin{algorithm}[H]
\caption{Priority rule algorithm}\label{algo:PRA}
\textbf{Input} \text{  List of all vehicles involved in the collision \((Lc) \),}\\ \text{\hspace{32pt} space\_time network routes (\(r_{s,t}\))} \\ 
\textbf{Output} \text{Best recommended priority ,} \\ \text{\hspace{41pt}updated space\_time network routes (\(r_{s,t}\)) } 

  \begin{algorithmic}[1]
  \State $P =   Perm(LC)$
  \State \text{best priority \(= None \)}
  \State \text{\(min\_val = \infty \)}
  
  \For{\(i \in P\)}
  \State \(sum\_energy = 0\)
  \For {\(j \in i\)}
  \State \text{Implement modified A* algorithm based on the priority \(i\)}
  \State \text{\(wasted\_time =\) wasted time for AGV \(j\) because of collision}
  \State 
\(sum\_energy\) =$ sum\_energy + \begin{cases}
200\times wasted\_time & \quad ,\text{not loaded} \\
500\times wasted\_time & \quad ,\text{loaded}
\end{cases}$

\EndFor
\State \If{\(sum\_energy < min\_val\)}
  \State best priority \(= i\)
  \State \(min\_val = sum\_energy\)
  \EndIf
  \EndFor

  \end{algorithmic}
\end{algorithm}

\subsection{Task assignment }
Task assignment plays a crucial role in the performance of any RMFS. The method used to allocate tasks to robots significantly impacts factors such as collision avoidance, energy consumption, and the total distance traveled by the robots.
The majority of the algorithms are based on the heuristic algorithms. These algorithms are usually based on neighborhood solution generation.
\subsubsection{Neighborhood solution generation}\label{NSG}
Neighborhood solution generation involves generating new candidate solutions that are in close proximity to an existing solution. This is achieved by making small, deliberate changes to the current solution. The core idea of the proposed task assignment algorithms is based on the idea of assigning tasks coming in the same aisle close to each other. By taking this idea, we can reduce the travel time of the robots and the possibility of collisions.
To implement this idea in generating neighbors, the list of tasks for each robot is represented by the number of aisle they belong to. For example, in Figure \ref{Task representation}, the aisles for the tasks assigned to  AGV \(V(A)\) are shown.   
\begin{figure}[h]
\centering
    \includegraphics[width= 3 in]{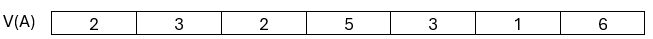}
\caption{Task representation}
\label{Task representation} 
\end{figure}

The specific operations for generating neighborhood solutions depend on the problem and the structure of the solution space. In this paper, swap, insertion and arrangement operations are implemented. These operations are discussed below.
The swap operation exchanges the positions of two elements in the solution. In this paper, swap operation is presented as follows. It begins by selecting two AGVs, \(A\) and \(B\). Next, it identifies the vehicle with the higher running time, denoted as \(A\). In AGV \(A\), it then locates the aisle with the highest task intensity and chooses a task assigned to vehicle \(B\) from that aisle. This task is swapped with a task from vehicle \(A\), ensuring that the selected task from vehicle \(A\) does not originate from the same high-intensity aisle.
Figure \ref{fig: swap} shows an example of swap operation. In this example, AGV \(A\) has a higher running time than vehicle \(B\). In AGV \(A\), we find the aisle with highest-intensity which is aisle \(1\). Then, we find the task in AGV \(B\) coming from same aisle, which is the first task of AGV \(B\). Then, we swap this task of AGV \(B\) with a random task of AGV \(A\) which is not coming from aisle \(1\)(first task of AGV \(A\)).

\begin{figure}[H]
     \captionsetup{justification=raggedright,singlelinecheck=false}
     \begin{subfigure}[b]{0.6\textwidth}
         \includegraphics[width=\textwidth,left]{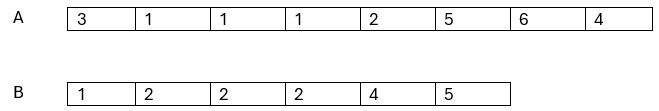}
         \caption{before swap }
         \label{fig: swap1}
     \end{subfigure}
     \hfill
     \begin{subfigure}[b]{0.6\textwidth}
         \includegraphics[width=\textwidth,left]{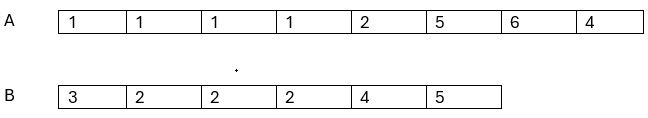}
         \caption{after swap}
         \label{fig:swap2}
     \end{subfigure}
     \hfill
        \caption{Swap operation}
        \label{fig: swap}
\end{figure}
The insertion operation removes an element from one position and inserts it into another. In this paper, insertion operation is presented as follows. It starts by selecting two AGVs, \(A\) and \(B\). It then identifies the AGV with the greater running time, designated as \(A\). Next, it determines the aisle with the highest task intensity assigned to vehicle \(B\) and selects a corresponding task from that same aisle in AGV \(A\). Finally, this task is removed from AGV \(A\) and inserted into a random location in AGV \(B\).
Figure \ref{fig: insertion} shows an example of insertion operation. In this example, AGV \(B\) has a lower running time than vehicle \(A\). In AGV \(B\), we find the aisle with highest-intensity which is aisle \(2\). Then, we find the task in AGV \(A\) coming from same aisle, which is the fifth task of AGV \(A\). Then, we remove this task from AGV \(A\) and insert it to the AGV \(B\).

\begin{figure}[H]
     \captionsetup{justification=raggedright,singlelinecheck=false}
     \begin{subfigure}[b]{0.6\textwidth}
         \includegraphics[width=\textwidth,left]{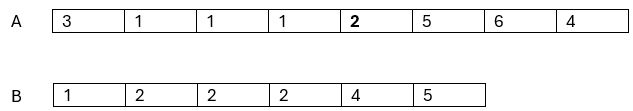}
         \caption{before insertion }
         \label{fig: insertion1}
     \end{subfigure}
     \hfill
     \begin{subfigure}[b]{0.6\textwidth}
         \includegraphics[width=\textwidth,left]{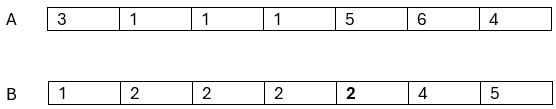}
         \caption{after insertion}
         \label{fig:insertion2}
     \end{subfigure}
     \hfill
        \caption{Insertion operation}
        \label{fig: insertion}
\end{figure}

In the arrangement operation, tasks assigned to an AGV are reorganized such that tasks originating from the same aisle are positioned consecutively. An example of the arrangement operation is seen in figure \ref{fig: arrangement}.

\begin{figure}[H]
     \captionsetup{justification=raggedright,singlelinecheck=false}
     \begin{subfigure}[b]{0.60\textwidth}
         \includegraphics[width=\textwidth,left]{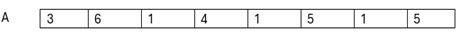}
         \caption{before arrangement }
         \label{fig: pre_arr}
     \end{subfigure}
     \hfill
     \begin{subfigure}[b]{0.6\textwidth}
         \includegraphics[width=\textwidth,left]{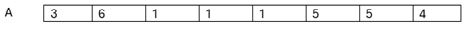}
         \caption{after arrangement}
         \label{fig: arr}
     \end{subfigure}
     \hfill
        \caption{Arrangement operation}
        \label{fig: arrangement}
\end{figure}

\subsection{Non-dominated sorting genetic algorithm}\label{NSGA}
Non-dominated sorting genetic algorithm (NSGA-II) is a popular method for solving multi-objective optimization problems \cite{zhou2021multi, shen2024nsga, keung2023cyber}. This algorithm starts with \(n\) initial solutions. The algorithm uses non-dominated sorting and crowding distance to rank solutions. Non-dominated sorting organizes solutions into different groups based on Pareto dominance. The first front includes solutions that are not outperformed by any others in the population, while later fronts contain solutions that are only outperformed by those in the earlier fronts. Crowding distance is used to evaluate solutions inside a front. Crowding distance calculates the proximity of a solution to other solutions in the front based on objective values, promoting diversity among solutions. Those with higher crowding distances are favored to achieve a broad range distribution along the front. After evaluating solutions, \(m\) (\(m<n\)) of the solutions are chosen to produce offspring. In this step, methods like crossover are applied to achieve this goal. The cycle continues until it reaches the termination criteria.

In this paper, we implement the NSGA-II algorithm with operations described in section \ref{NSG} to improve the proposed approach. Specifically, during the offspring generation phase, these operations are combined with genetic algorithm techniques to produce new offspring. This strategy aids the algorithm in directing its search more effectively to discover better solutions. This algorithm can be seen in algorithm \ref{algo:NSGA}. In the proposed NSGA-II algorithm, After selecting best \(m\) of the solutions, half of them are chosen for generating new offspring based on the crossover operations in genetic algorithms. Other half of the solutions are chosen for generating new offspring based on the proposed operations like swap, insertion and arrangement operations.    

In the genetic algorithm, each solution is represented as a chromosome. This chromosome has two sections. The first section shows the tasks of the solution and second section contains the number of tasks assigned to each AGV in that solution. A sample chromosome can be seen in figure \ref{ch_rep}.

\begin{figure}[H]
\centering
    \includegraphics[width= 4 in]{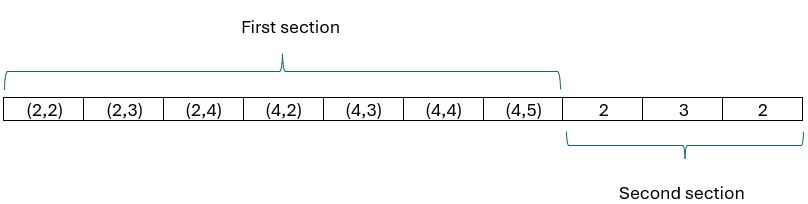}
\caption{Chromosome representation }
\label{ch_rep} 
\end{figure}
In this figure, the first section represents the location of each task and the second section shows number of tasks for each AGV. In this example, there are three AGVs in the system in which first two tasks belong to first AGV, the next three tasks belong to second AGV and the last two tasks belong to third AGV. After selecting parents, two different crossover techniques are implemented for generating offsprings. 

The first crossover technique is partially matched crossover (PMX) which is implemented on the first section of the chromosomes. It is a technique to mix the parents with each other while keeping the orders valid. It exchanges a segment of the sequence between parents and resolves any duplicate values through a mapping process.

The second crossover technique is arithmetic crossover (ARC) which is implemented on the second section of the chromosome. This technique is used for determining number of tasks for each AGV. ARC algorithm can be seen in equation \ref{equal:ARC}.

\begin{equation}\label{equal:ARC}
X_{i}^k = \floor{\alpha\times X_{i}^k + (1-\alpha)\times X_{j}^k} 
\end{equation}
In equation \ref{equal:ARC}, \(X_{i}^{k}\) and \(X_{j}^{k}\) represent the gene values at position \(k\) of chromosomes \(i\) and \(j\), respectively. The parameter \(\alpha\) is a weighted value that ranges between 0 and 1.

The ARC algorithm is applied on the second section of the parents' chromosomes to determine the number of tasks assigned to each child based on the tasks from the parent solutions. If the total number of tasks assigned to the children exceeds the total available tasks, the vehicle with the highest number of tasks is identified, and its task count is reduced by one. This process is repeated until the total tasks assigned to the children match the available tasks. Conversely, if the total number of tasks assigned to the children is less than the total available tasks, the vehicle with the fewest tasks is identified, and its task count is increased by one. This process continues until the total tasks assigned to the children equal the available tasks.

\begin{algorithm}[H]
\caption{NSGA algorithm}\label{algo:NSGA}
\textbf{Input} \text{ List of tasks(\(Tasks\)), Number of AGVs (\(N_{AGV}\)) }\\ \text{\hspace{32pt} Total number of running for algorithms($T\_R$) ,}\\ \text{\hspace{32pt} Population size $(P\_S)$ }  \\ 
\textbf{Output} \text{Best recommended set of task assignments }  

  \begin{algorithmic}[1]
  \State $I\_S = $  \text{Randomly assign set of tasks to AGVs }
  
  \State $i = 0$
  \While{\(i < T\_R\)}
  \State \parbox[t]{.6\linewidth}{Finding set of goals $(G1,G2,G3,G4)$ for $I\_S$}
  
  \State \parbox[t]{.6\linewidth}{Implementing Non-dominating sorting and crowding distance on the set of goals  }
  \State \text{Choosing \(m(m < n)\) best of solutions}
  \State \text{$m_1 = $} \parbox[t]{.6\linewidth} {Choosing $   \floor{m/2}$ of \(m\) and applying partially matched crossover to 
  generate offspring }

  \State \text{$m_2 = $ } \parbox[t]{.6\linewidth}{Choosing solutions not chosen in line 7 and implementing operations of \ref{NSG} on them to generate offspring}
  \State \text{$I\_S = m_1 \cup m_2$}
  \State $i = i+1$

  \EndWhile
  \State \textbf{Return } \text{$I\_S$}

  \end{algorithmic}
\end{algorithm}

\subsection{Adaptive large neighborhood search}
Large neighborhood search (LNS) is an evolutionary algorithm which is iteratively used for improving solutions. This algorithm consists mainly of destroy and rebuild operations. In destroy, a part of the solution is destroyed. Then, in rebuild, that part is rebuilt so that it leads to a new solution. This method enables the algorithm to escape from local optima and investigate a wider range of the solution space. This algorithm has been very popular and has been implemented in assignment problem\cite{fathollahi2023efficient}, routing problem\cite{le2023small}, and task planning\cite{wang2024hybrid}.

Adaptive large neighborhood search (ALNS) is an extension of LNS. In this algorithm, the operations which are more suitable for the problem have higher priority than other operation. In this way, the algorithm can be guided to find the better solutions than usual LNS algorithms. ALNS has been widely used in different areas like vehicle routing problem\cite{dumez2021large}, travelling salesman problem\cite{smith2017glns}, and share-a-ride problem\cite{li2016adaptive}.

In this paper, an ALNS algorithm is proposed. This is achieved by integrating LNS with operation methods introduced in \ref{NSG}. In this algorithm, two sets of operations are used for generating neighbor solutions: methods in \ref{NSG} and random solution generation methods.
Random solution methods, such as swap, insertion, and arrangement operations, are implemented without predefined guidelines, relying on randomization to generate solutions rather than following the structured approaches detailed in Section \ref{NSG}. 
In the next step, an adaptive mechanism is implemented to find the best solution. The adaptive mechanism combines various proposed operations with a probabilistic selection approach, like roulette wheel selection, to dynamically enhance the search process. This approach increases the possibility of selecting operations that have demonstrated strong performance, maximizing their effectiveness. At the same time, it preserves the possibility to choose less frequently used operations, promoting exploration and minimizing the likelihood of the search becoming stuck in local optima.

 The algorithm can be seen in algorithm \ref{algo:ALNS}. In this algorithm, six operations are used: random swap, random insertion, random arrangement, guided swap, guided insertion, guided arrangement. The first three operation are the random operations and the second three ones are the operations discussed in section \ref{NSG}. At the beginning, all six operations have the same weight and the possibility of choosing each one is equal \(1/6\). In first iteration, neighbors of the current solution are generated based on the operations. After generating all solutions, the best one is chosen and the the weight of the corresponding operation would be updated. This operation has a higher weight and is more probable to be chosen for generating new solutions in the next iteration. This process repeats in each iteration. In line \(16\) of the algorithm, solution \(i\) dominates solution \(cur\_best\) means that \(i\) is at least as good as \(cur\_best\) in all objectives and strictly better in at least one objective.  

\begin{algorithm}[H]
\caption{ALNS algorithm}\label{algo:ALNS}
\textbf{Input} \text{ List of tasks(\(Tasks\)), Number of AGVs (\(N_{AGV}\)) }  \\ 
\textbf{Output} \text{Best recommended task assignment }  

  \begin{algorithmic}[1]
  \State $I\_S = $  \text{Randomly assign set of tasks to AGVs }
  
  \State $done = False$
  \State \text{$R\_w = $ } \parbox[t]{.8\linewidth}{Initialize weights for operators} 
  \While{\(done \)}
  \State  \text{$m_1 =$} \parbox[t]{.8\linewidth} {Implement operations in section \ref{NSG} and random operations to  generate neighbors}
  \State $cur\_best = I\_S$
  \State $temp = True$
  \For{ $i \in m_1$}
  \If{$i$ \text{dominates} $cur\_best$}
  \State $cur\_best = i$
  \State $temp = False$
  
  \EndIf
  \EndFor

  \If{temp}
  \State \text{break}
  \EndIf

  \State $I\_S = cur\_best$
  \State \text{Update $R\_w$ }

  \EndWhile
  \State \textbf{Return } \text{$I\_S$}

  \end{algorithmic}
\end{algorithm}
\section{Results}\label{resul}
\subsection{Performance evaluation}
The metrics which we consider for evaluation are \textit{hypervolume (HV)}, and \textit{covered size of space (CSS)}.
\begin{itemize}

    \item The \textit{hypervolume (HV)} indicator represents the volume of the region in the objective space that is covered by the Pareto front approximation \(X_{N}\) and bounded above or below by a reference objective point \textit{r}. A sample of hypervolume calculation for a bi-objective optimization problem can be seen in figure \ref{HV representation}. Indeed, the higher the hypervolume value, the better the performance of the Pareto front, when all objectives are being minimized.

\begin{figure}[H]
\centering
    \includegraphics[width= 2.5 in]{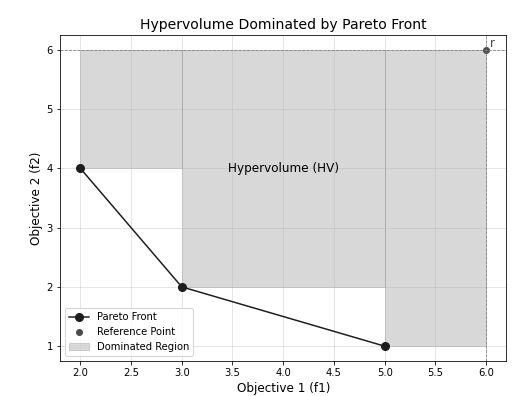}
\caption{Hypervolume sample }
\label{HV representation} 
\end{figure}
\item Ziztler et al.\cite{zitzler2000comparison} first introduced the \textit{covered size metric (CSS)} to compare the Pareto front solutions of two sets with each other. Given two sets \(X,Y\), the \textit{CSS} of these sets, shown as \(C(X,Y)\), computes the ratio of Pareto front solutions in set \(Y\) that are dominated by at least one solution in Pareto front of set \(X\). \(C(X,Y) = 0\) means that there is no solution in set \(Y\) that is dominated by any Pareto front solutions in set \(X\), while \(C(X,Y) = 1\) means that every solution in set \(Y\) is dominated by at least one solution in Pareto front of set \(X\). 

\begin{equation}\label{equal:CSS}
C(X,Y) =  \frac{|a_{Y}\in Y; \exists a_{X}\in X,a_{X}\leq a_{Y} |}{|Y|} 
\end{equation}

\end{itemize}

\subsection{Experimental results} 
\subsubsection{Priority rules}
The results for the priority rules are shown for thirteen scenarios in table \ref{tab:PR}. In this table, \textit{T} represents the number of tasks, and \textit{V} represents number of AGVs. For each scenario, one set of tasks is subjected to both scenarios for \(50\) times (\(50\) different combinations). For each set of tasks, the path routing algorithm for both benchmark and proposed priority rule are performed and the average values are calculated and can be seen in table \ref{tab:PR}.
\begin{table}[h]
\begin{center}
\caption{Summary of priority rules Comparison}
\label{tab:PR}
\begin{tabular}{c c c c}
   \hline
   |\(|T|\)|  &  |\(|V|\)|  & \multicolumn{1}{c}{Benchmark} & \multicolumn{1}{c}{Proposed}\\ [3pt]

       &   &   HV & HV   \\[2pt]
   \hline
   36 & 5 &3.550& \textbf{3.556}  \\[1pt]
   72	&5 & 33.580& \textbf{33.640}\\[1pt]
   100	&5 & 90.642& \textbf{90.778}\\[1pt]
   36	&8 & 0.786& \textbf{0.790}\\[1pt]
   72	&8 & 10.000& \textbf{10.040}\\[1pt]
   100	&8 & 31.788& \textbf{31.892}\\[1pt]
   36	&10 & 0.700& \textbf{0.702}\\[1pt]
   72	&10 & 6.256& \textbf{6.276}\\[1pt]
   100	&10 & 16.040& \textbf{16.122}\\[1pt]

\end{tabular}
\end{center}
\end{table}

\subsubsection{Task assignment algorithms}

The results for the comparison of the task assignment algorithms are shown in table \ref{tab:CA}. The results are for thirteen different scenarios. For each scenario, \(30\) different tasks combinations are processed. Each task combination has been run for \(20\) minutes. The results in the table shows the average of the Pareto front solutions. For each algorithm, results are provided for the proposed and benchmark versions. Also, for each version, \textit{HV} are calculated.      

\begin{table}[H]
\begin{center}
\caption{Summary of task assignment algorithms performance}
\label{tab:CA}
\begin{tabular}{c c c c c c }
   \hline
   |\(|T|\)|  &  |\(|V|\)|  & {ALNS\_benchmark} & {ALNS\_improved} & {NSGA-II\_benchmark} & {NSGA-II\_improved}\\ [3pt]

       &  &  HV & HV & HV &  HV   \\[2pt]
   \hline
   36 & 5 & 	0.102& \textbf{0.106}&			\textbf{0.118}& 0.117\\[1pt]
   72	&5 & 	1.991& \textbf{2.939}&	4.479 &\textbf{5.073}\\[1pt]
   100	&5 & 	5.128&	 \textbf{9.706}& 9.740&\textbf{13.532}\\[1pt]
   36	&8 & 	0.195 & \textbf{0.229}&	0.331& \textbf{0.701}\\[1pt]
   72	&8 & 	1.057& \textbf{1.448}& 1.683  &\textbf{1.862}\\[1pt]
   100	&8 &  	3.380& \textbf{5.084}&	5.940 &\textbf{6.720}\\[1pt]
   36	&10 & 	0.430& \textbf{0.440}&0.550 &\textbf{0.620}\\[1pt]
   72	&10 & 	8.704& \textbf{10.002}& 12.484 &\textbf{12.655}\\[1pt]
   100	&10 & 	2.212& \textbf{3.016}&	3.902&\textbf{4.226}\\[1pt]

\end{tabular}
\end{center}
\end{table}

The figure \ref{fig:TA_graph} shows the results of \(HV\) indicator for each algorithm in details. Figures \ref{fig: 5_36_TA}, \ref{fig: 8_36_TA} and \ref{fig: 10_36_TA} shows the results for 5, 8 and 10 AGVs, respectively.

\begin{figure}[H]
     \centering
     \begin{subfigure}[b]{1.1\textwidth}
         \includegraphics[width=\textwidth,left]{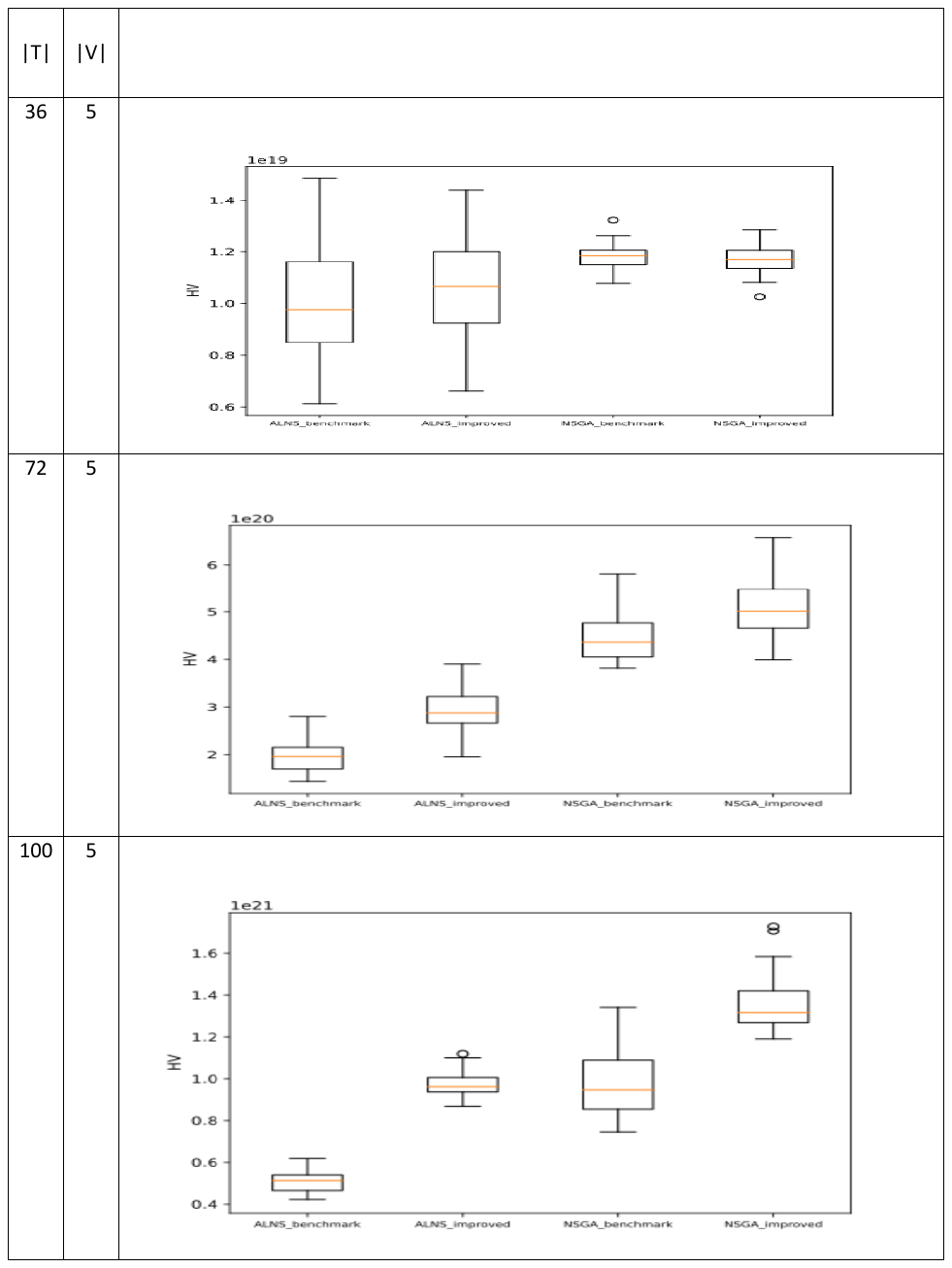}
         \caption{5 AGVs }
         \label{fig: 5_36_TA}
     \end{subfigure}
     \end{figure}
\clearpage

\begin{figure}[H]
\ContinuedFloat
     \centering
     \begin{subfigure}[b]{1.1\textwidth}
         \includegraphics[width=\textwidth,left]{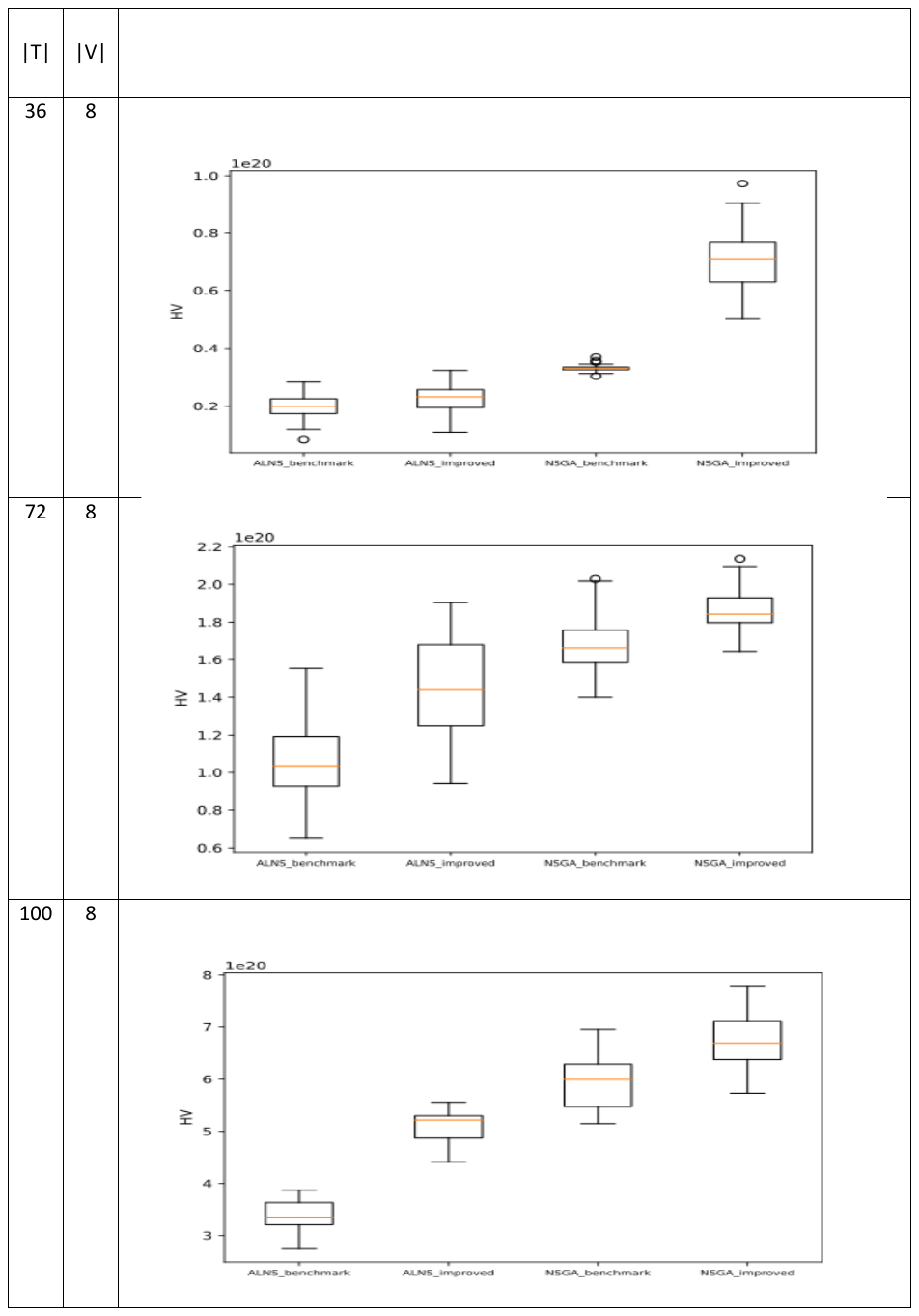}
         \caption{8 AGVs}
         \label{fig: 8_36_TA}
     \end{subfigure}

\end{figure}
\begin{figure}[H]
\ContinuedFloat
     \centering
     \begin{subfigure}[b]{1.1\textwidth}
         \includegraphics[width=\textwidth,left]{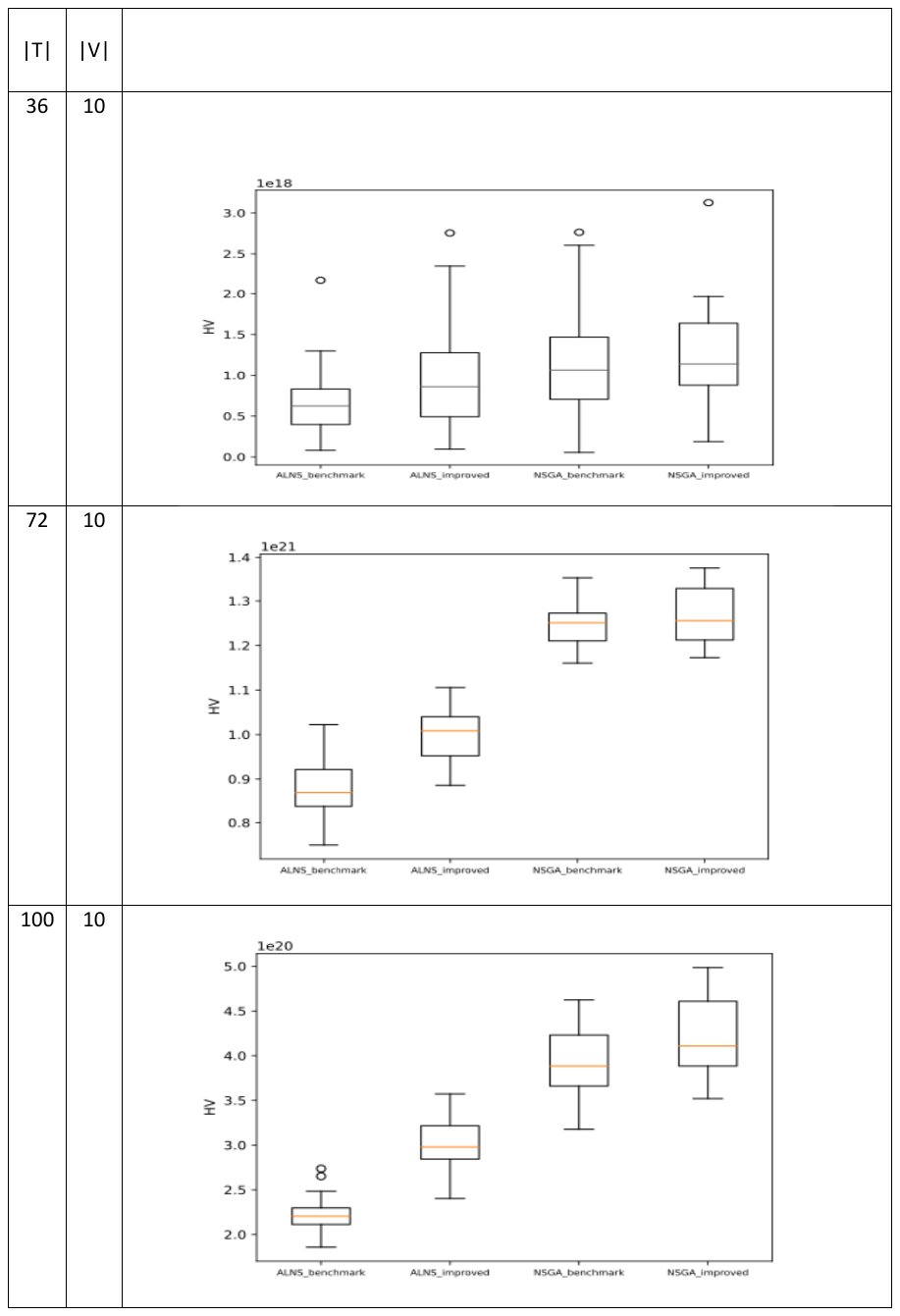}
         \caption{10 AGVs}
         \label{fig: 10_36_TA}
     \end{subfigure}
        \caption{Task assignment algorithms}
        \label{fig:TA_graph}
\end{figure}
The computational results indicated that the proposed NSGA-II and ALNS algorithms both outperform the benchmark NSGA-II and ALNS. The Kruskal-Wallis test has also been performed on the algorithms and it has confirmed that the performance of these algorithms come from different distributions. Comparing the NSGA and ALNS algorithms, it can be seen that NSGA-II has a better performance than ALNS, in both proposed and benchmark methods. We hypothesize that the reason is the proposed ARC technique for the NSGA-II algorithm. This approach helps the NSGA-II algorithm to keep a better balance of tasks among AGVs compared to the ALNS algorithm.

The second metric that we have considered is \textit{CSS} metric. The results of this metric can be seen in table \ref{tab:Csstot}. In this table, entries in row \(i\) and column \(j\) indicates the CSS metric comparing algorithm \(i\) to algorithm \(j\) for each scenario. It is observable that algorithm NSGA-II\_improved dominates the other algorithms. For example, for 8 AGVs, NSGA-II\_improved algorithm dominates all other algorithms more than \(99\) percent. The second best algorithm is NSGA-II\_benchmark which dominates ALNS algorithms more than \(90\) percent, on average. Regarding ALNS algorithms, ALNS\_improved completely dominates ALNS\_benchmark.  

\begin{table}[H]
\caption{CSS ratio performance of algorithms}
\label{tab:Csstot}
\begin{tabular}{ccccccc}

\hline
AGVs & Tasks & Algorithms  & ALNS\_bench & ALNS\_imp & NSGA-II\_bench & NSGA-II\_imp \\
\hline
5    & 36    & ALNS\_bench &       -      & 0.85      & 0.44        & 0.70       \\
     &       & ALNS\_imp   & 0.92        &   -        & 0.91        & 0.83      \\
     &       & NSGA-II\_bench & 0.80         & 0.85      &      -       & 1.00         \\
     &       & NSGA-II\_imp   & 0.80         & 0.83      & 0.86        &      -     \\ [7pt]
     & 72    & ALNS\_bench & -            & 0.38      & 0           & 0         \\
     &       & ALNS\_imp   & 1.00           &   -        & 0           & 0         \\
     &       & NSGA-II\_bench & 1.00           & 1.00         &       -      & 0.88      \\
     &       & NSGA-II\_imp   & 1.00           & 1.00         & 1.00           &       -    \\ [7pt]
     & 100   & ALNS\_bench &   -          & 0.01      & 0           & 0         \\
     &       & ALNS\_imp   & 1.00           &     -      & 0.18        & 0         \\
     &       & NSGA-II\_bench & 1.00           & 0.93      &     -        & 0.3       \\
     &       & NSGA-II\_imp   & 1.00           & 1.00         & 1.00           &     -      \\
     \hline
8    & 36    & ALNS\_bench &      -       & 0.57      & 0           & 0         \\
     &       & ALNS\_imp   & 1.00           &   -        & 0           & 0         \\
     &       & NSGA-II\_bench & 0.75        & 0.53      &      -       & 0         \\
     &       & NSGA-II\_imp   & 1.00           & 1.00         & 1.00           &      -     \\ [7pt]
     & 72    & ALNS\_bench &       -      & 0.63      & 0           & 0         \\
     &       & ALNS\_imp   & 1.00           &   -        & 0.07        & 0.04      \\
     &       & NSGA-II\_bench & 1.00           & 0.93      &      -       & 0.91      \\
     &       & NSGA-II\_imp   & 0.99        & 1.00         & 0.98        &      -     \\ [7pt]
     & 100   & ALNS\_bench &       -      & 0.12      & 0           & 0         \\
     &       & ALNS\_imp   & 1.00           &  -         & 0.03        & 0.02      \\
     &       & NSGA-II\_bench & 1.00           & 0.93      &      -       & 0.76      \\
     &       & NSGA-II\_imp   & 1.00           & 1.00         & 0.98        &    -       \\
     \hline
10   & 36    & ALNS\_bench &       -      & 0.52      & 0           & 0         \\
     &       & ALNS\_imp   & 1.00           &   -        & 0.01        & 0         \\
     &       & NSGA-II\_bench & 1.00           & 1.00         &       -      & 0.75      \\
     &       & NSGA-II\_imp   & 1.00           & 1.00         & 1.00           &      -     \\ [7pt]
     & 72    & ALNS\_bench &      -       & 0.46      & 0           & 0         \\
     &       & ALNS\_imp   & 1.00           &  -         & 0           & 0.04      \\
     &       & NSGA-II\_bench & 1.00           & 1.00         &       -      & 0.90       \\
     &       & NSGA-II\_imp   & 1.00           & 1.00         & 0.98        &      -     \\ [7pt]
     & 100   & ALNS\_bench &     -        & 0.22      & 0           & 0         \\
     &       & ALNS\_imp   & 1.00           &   -        & 0           & 0         \\
     &       & NSGA-II\_bench & 1.00           & 1.00         &     -        & 0.89      \\
     &       & NSGA-II\_imp   & 1.00           & 1.00         & 0.98        &      -    
\end{tabular}

\end{table}

\section{Managerial insights}\label{MI}
This paper focuses on the energy use of AGVs as an important factor in how well robotic systems work. To the best of the authors' knowledge, no previous research has studied how energy is used when solving collisions between AGVs in warehouses. The findings of this study offer managers a new way to think about AGV collisions, highlighting the potential to save energy in these situations. Even small energy savings can reduce how often AGVs need to recharge their batteries, which can save both time and money. The task of evaluating how this impacts real-world operations is left to managers and decision-makers in the industry.

Another focus of this paper is on how tasks are assigned to AGVs. In practice, managers may not fully realize how much task assignment affects the system’s performance, such as the time AGVs spend running. This study combines simple ideas with existing algorithms to make task assignments better. For example, the approach suggests assigning tasks in the same aisle to the same AGV, which helps reduce the distance AGVs need to travel and lowers the chance of collisions. By using the proposed methods, managers can improve how efficiently their systems work and make better use of AGVs.   

\section{Conclusion}\label{CONC}
This paper starts with emphasizing the importance of collision avoidance in robotic mobile fulfillment systems (RMFSs). The study sheds light on the less-studied area of energy consumption in the RMFS. We develop an algorithm incorporating running time and energy consumption of the AGVs, which is a function of loading status of the AGVs. This algorithm helps to set up a novel approach to foster the energy consumption and running time of the AGVs. The results shows the superiority of the proposed algorithm with the benchmark method.

Regarding task assignment, novelty comes from two parts: \((i)\) considering energy consumption of the AGVs. \((ii)\) proposed methods for assigning tasks to the AGVs. This paper first incorporates energy consumption with the existing objectives for performance evaluation of the system. Then, it compare the proposed methods with existing methods. Two algorithms are chosen to implement methods on them: NSGA-II and ALNS algorithms. NSGA-II is a multi objective algorithm and ALNS has been modified to include multiple objectives.  

The results showcase the power of proposed methods. For both algorithms, proposed methods dominates the existing methods overwhelmingly. In matter of comparing NSGA-II with the ALNS, it can be seen that NSGA-II works better. It stems from the proposed ARC operator which is introduced determining number of tasks for each AGV. This proposed ARC tries to keep balance of number of tasks for AGVs while exploring the solution space.

Future work could extend the current model to handle real-time task arrivals and energy-aware scheduling by incorporating dynamic order streams, battery constraints, and charging policies into the optimization framework.  

\section{Acknowledgments}
The authors acknowledge Clemson University for generous allotment of computational resources on Palmetto cluster.
\section{Data availability statement}
Data is available on request from authors.
\section{Disclosure statement}
There are no conflicts of interest.
\section{Funding}
This material is based upon work supported by the National Science Foundation under Grant No. 2322250.
\section{Notes on Contributors}
\hangindent=\dimexpr 1in+1em\relax
\hangafter=-8
\noindent\llap{\raisebox{\dimexpr\ht\strutbox-\height}[0pt][0pt]{\includegraphics[width=1in,height=1.1in]{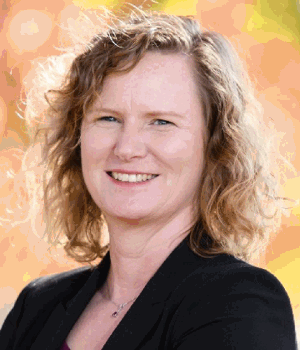}}%
    \hspace{1em}}
\textbf{Mary E. Kurz,} PhD, is an Associate Professor in Industrial Engineering at Clemson University. She is a member of INFORMS and a Senior member of the Institute of Industrial Engineers. Dr. Kurz received her doctoral degree from the University of Arizona in Systems and Industrial Engineering in 2001, focusing on scheduling flexible flowlines. Dr. Kurz is interested in improving the efficiency of large assembly systems through better integration of various functions, such as sequencing, line balancing, and material handling. She develops mathematical models, heuristics and metaheuristics to evaluate the potential for system improvements.
Mary E. Kurz holds concurrent appointments as an associate professor at Clemson University and as an Amazon Scholar. This manuscript describes work performed at Clemson University and is not associated with Amazon.\newline

\hangindent=\dimexpr 1in+1em\relax
\hangafter=-8
\noindent\llap{\raisebox{\dimexpr\ht\strutbox-\height}[0pt][0pt]{\includegraphics[width=1in,height=1.1in]{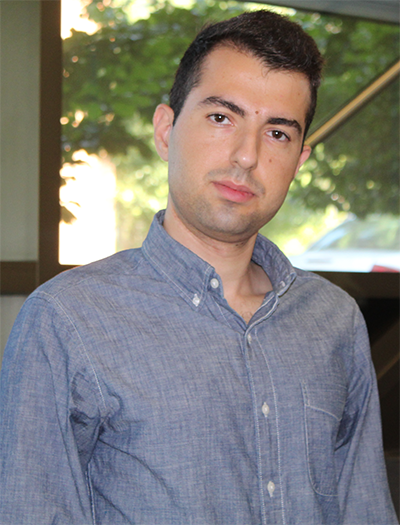}}%
    \hspace{1em}}
\textbf{Ahmad Kokhahi,} is a PhD student at Clemson University. He earned his bachelor degree from AmirKabir University, Iran in 2017. He earned his master's degree from Sharif University of Technology, Iran in 2020. His doctoral research focuses on warehouse optimization, particularly in RMFS within e-commerce environments. His work integrates heuristic task allocation, dynamic path planning, and energy-efficient collision avoidance strategies, achieving significant improvements in throughput, cost reduction, and operational efficiency.

\bibliographystyle{unsrt}
\bibliography{references}

\end{document}